\documentclass[lettersize,journal]{IEEEtran}
\usepackage{amsmath,amsfonts}
\usepackage{algorithmic}
\usepackage{algorithm}
\usepackage{array}
\usepackage[caption=false,font=normalsize,labelfont=sf,textfont=sf]{subfig}
\usepackage{textcomp}
\usepackage{stfloats}
\usepackage{url}
\usepackage{verbatim}
\usepackage{graphicx}
\usepackage{cite}
\usepackage{multirow}
\usepackage{fancyhdr}
\usepackage[hang,flushmargin]{footmisc}

\usepackage{hyperref}       
\usepackage{booktabs,tabularx}       

\usepackage{nicefrac}       
\usepackage{breqn}

\usepackage{adjustbox}
\usepackage{amssymb}
\usepackage{amsthm}

\usepackage{xcolor}
\usepackage{cleveref}
\usepackage{subfig}

\newtheorem{theorem}{Theorem}
\newcolumntype{C}{>{\centering\arraybackslash}X} 
\usepackage{lipsum} 

\def\reals{\mathbb{R}}

\def\comp{\raise 1pt \hbox{$\scriptstyle\circ$}}

\def\argmin{\mathop{\rm argmin}\limits}

\def\minimize{\mathop{\rm min}\limits}

\def\st{\mathop{\rm subject\ to}}

\def\upto{{\raise 1pt \hbox{$\scriptstyle \,\nearrow\,$}}}
\def\downto{{\raise 1pt \hbox{$\scriptstyle \,\searrow\,$}}}

\def\tos{\rightrightarrows}

\makeatletter
\newcommand\notsotiny{\@setfontsize\notsotiny\@vipt\@viipt}
\makeatother

\begin{document}

\title{A Globally Convergent Gradient-based Bilevel Hyperparameter Optimization Method}

\author{Ankur Sinha, Satender Gunwal, and Shivam Kumar
\thanks{A. Sinha, S. Gunwal, and S. Kumar are with the Centre for Data Science and AI, Indian Institute of Management Ahmedabad. ( e-mail: asinha@iima.ac.in, satender@iima.ac.in, and shivam.kumar@qi-cap.com, respectively.)} 
\thanks{ \textit{Corresponding author : Satender Gunwal}}
}

\markboth{}%
{Shell \MakeLowercase{\textit{et al.}}: A Globally Convergent Gradient-based Bilevel Hyperparameter Optimization Method}


\maketitle

\begin{abstract}
Hyperparameter optimization in machine learning is often achieved using naive techniques, such as random and grid search, that only lead to an approximate set of hyperparameters. Although techniques such as Bayesian optimization perform an intelligent search, it does not guarantee an optimal solution. A major drawback of most of these approaches is that as the number of hyperparameters increases, the search domain increases exponentially, thus increasing the computational cost. The hyperparameter optimization problem is inherently a bilevel optimization task, and some studies have attempted bilevel solution methodologies for solving this problem. However, these studies, too, suffer from a drawback, where they assume a unique set of model weights that minimize the training loss, which is generally violated by deep learning architectures. This paper discusses a gradient-based bilevel method for the continuous hyperparameter optimization problem, where these drawbacks have been addressed. The method guarantees convergence to the set of optimal hyperparameters that this study has theoretically proven. The idea is based on approximating the lower-level optimal value function mapping using Gaussian process regression to obtain a single-level formulation, which is then solved using the augmented Lagrangian method. We have performed an extensive computational study on the MNIST and CIFAR-10 datasets on multi-layer perceptron and LeNet architectures that confirms the efficiency of the proposed method. A comparative study against grid search, random search, Bayesian optimization, and HyberBand method on multiple hyperparameter problems shows that the proposed algorithm converges with lower computation and leads to models that generalize better on the testing set.
\end{abstract}


\section{Introduction}
\IEEEPARstart{T}{raining} a machine learning model that performs better on unseen data requires hyperparameter tuning, which is a computationally expensive task. It is often done using non-exact methods where a search space is given to the model for identifying a good set of hyperparameters. Methods like grid search, used in practice, set a grid of hyperparameter values beforehand and evaluate models corresponding to each hyperparameter on a validation set. Random search \cite{random_search:James} is an alternative to grid search where randomly generated hyperparameters are used to generate models using the training set, which are then evaluated on the validation set. Bayesian optimization \cite{PredModel:Kuhn,BayesianHO:Snoek} is a more intelligent approach that carries out the hyperparameter optimization by assuming a prior. Bayesian optimization also samples hyperparameters in the search space, evaluates them, and updates the prior to form the posterior distribution. There is a significant body of literature on hyperparameter optimization in machine learning, see \cite{ML_Book:HO,HOReview:TongYu,HO:Yang,optML:Claudio}.
The problem of optimizing the hyperparameters is inherently a \textit{bilevel optimization} problem, as one can clearly observe that all the above methods create models corresponding to different hyperparameters, which entails solving multiple optimization problems for model training. The various models trained are then evaluated based on validation data, and the best performing model is chosen. Model training using the training data for a given hyperparameter is the lower level optimization task and the model selection based on validation performance is the upper level optimization task. The set of optimal hyperparameters prevents overfitting on the training data and results in a more generalizable model \cite{HO:Yang,GenBilevel:FBao}. For a review on bilevel optimization, refer to \cite{dempe2002foundations,bilevel_review:asinha}. 
Other state-of-the-art methods that follow the similar idea of probabilistic modelling as in Bayesian optimization include Sequential Model-based Global Optimization \cite{hutter2011sequential}, Gaussian Process Approach \cite{BayesianHO:Snoek,GPforML}, and Tree-structured Parzen Estimator Approach \cite{bergstra2011algorithms}. For most of these approaches, the computational time increases exponentially, and the performance deteriorates significantly as the number of hyperparameters increase \cite{GB_HO:Maclaurin}.

It is well-known that gradient-based methods can optimize a significantly large number of hyperparameters \cite{domke2012generic,HO_Gradient:Fabian,GB_HO:Maclaurin,Fu_et_al:2016}. Some aspects of our study fall into this category of hyperparameter optimization as we develop a gradient based approach that simultaneously optimizes the hyperparameters and model weights in a bilevel framework. Other studies in this area include \cite{bengio2000gradient} and \cite{HO_Gradient:Fabian}, which uses implicit function theory for the derivation of hyperparameter gradients. Another category of similar papers uses an iterative approach to approximate the reaction set mapping after solving a few lower level optimization problems \cite{franceschi2018bilevel, GB_HO:Maclaurin}. Most of the bilevel methods that are based on approximation of the reaction set mapping, i.e. the optimal set of model weights for a given hyperparameter, assume the existence of a unique optimum at the lower level that is often violated by deep learning structures. Recently, \cite{lorraine2018stochastic} and \cite{mackay2019self} proposed similar methods that locally approximate this mapping by estimating the hessian or its inverses, making their methods computationally intractable for a large number of model parameters.


Our proposed method relies on approximating the optimal value function (OVF) mapping that does not require the assumption of a unique optimum at the lower level. The optimal value function is approximated using a stochastic process model to convert the bilevel program into a constrained single level problem. Further, the constrained problem is reduced to an unconstrained optimization problem using the augmented Lagrangian method, which we solve using a gradient-based approach. \cite{Gradient_Based_Past} recently proposed a gradient-based bilevel approach for hyperparameter tuning where the optimal value function is approximated in one shot using Kriging approximation. In the present work, the optimal value mapping is improved iteratively using the newly obtained upper level solution after each augmented Lagrangian step. We also consider the standard error arising from the stochastic process model to ensure a tighter confidence interval of the approximation. Assuming that we have a global optimizer for unconstrained optimization problems, our method guarantees convergence to the globally optimal solution (i.e., the optimal hyperparameters and the corresponding model weights) to the bilevel problem. We have provided our algorithm's termination conditions and convergence results in this paper.

An extensive comparative study against the common hyperparameter optimization approaches, like grid search, random search, Bayesian optimization \cite{hutter2011sequential} and HyperBand \cite{hyperband} method on MNIST \cite{lecun2010mnist} and CIFAR-10 \cite{Krizhevsky09learningmultiple} datasets reveals that the proposed approach is able to search models with high generalization at significantly lower computations. The MNIST dataset has been used with the multi-layer perceptron architecture (MLP), and the CIFAR-10 dataset has been used with the convolutional neural network (CNN) architecture. In this study, we have presented results for problems involving up to 4 regularization parameters. The proposed method can easily handle a higher number of regularization hyperparameters without a deterioration in performance or a significantly higher requirement of computational resources. However, we have restricted ourselves to 4 hyperparameters as it would not be easy to draw comparisons against other methods as their computational requirements increase substantially with an increase in hyperparameters.

The paper is organized as follows. Section~\ref{section:2} provides a bilevel optimization framework for hyperparameter optimization in the context of machine learning. Further, it discusses some single level reduction approaches that can be used to solve such problems and their drawbacks. Section~\ref{section:3} discusses about the proposed bilevel approach and its termination and convergence proofs in detail. The results of our studies on MNIST and CIFAR-10 datasets are discussed in Section~\ref{section:4}. Section~\ref{section:5} provides the conclusions of our study.

\section{Bilevel Optimization Framework}\label{section:2}

\begin{table*}
\caption{central notations used in the paper}
\centering
\label{Table1:Notations}
\begin{tabularx}{500pt}{>{\hsize=.5\hsize\linewidth=\hsize}X >{\hsize=.5\hsize\linewidth=\hsize}X >{\hsize=2\hsize\linewidth=\hsize}X}
\toprule
\textbf{Category} & \textbf{Symbols Used} & \textbf{Description} \\ 
\midrule

\multirow{2}{*}{Datasets} & $D_{tr}$ & Training set; $\{(x_i,y_i)\}_{i = 1}^{n_1}$, with $n_1$ training examples\\ [3pt]
& $D_{val}$ & Validation set; $\{(x_i,y_i)\}_{i = 1}^{n_2}$, with $n_2$ validation examples\\
\midrule

\multirow{2}{*}{Variables} & $\lambda \in \mathbb{R}^n$ & Model hyperparameters (upper level variables)  \\ [3pt]
  & $\beta \in \mathbb{R}^m$ & Model parameters (lower level variables)\\
\midrule

Loss Function & $L(\beta; D)$ & Loss-function on a given example set $D$ with model parameters $\beta\in \mathbb{R}^m$\\
\midrule

Regularization & $\Theta(\lambda, \beta)$ & $\lambda\|\beta\|^2$ (L2 regularization)\\
\midrule

\multirow{2}{*}{Objective Functions} & $F(\beta)$ &  $L(\beta; D_{val})$; upper level objective ( validation Loss )\\ [3pt]
  &  $f(\lambda, \beta)$ &  $L(\beta; D_{tr}) + \Theta(\lambda, \beta)$; lower level objective ( regularized training loss ) \\
\midrule

Reaction Set &  $\Psi(\lambda): \mathbb{R}^{n} \rightarrow \mathbb{R}^{m}$ & $\Psi(\lambda)$ represents optimal lower level solutions for any given upper level variable $\lambda \in \mathbb{R}^n$ \\
\midrule

Optimal Value Function & $\phi(\lambda): \mathbb{R}^n\longrightarrow \mathbb{R}$ & Optimal objective value of the lower level problem for any given upper level variable $\lambda \in \mathbb{R}^n$ \\
\bottomrule
\end{tabularx}
\end{table*}

In bilevel optimization problems, there are two levels of optimization, generally referred to as upper (or outer) level optimization and lower (or inner) level optimization. The lower level problem is a nested optimization task that acts as a constraint to the upper level problem. Note that the presence of multiple optimal solutions at the lower level may give rise to two different kinds of formulations; namely, optimistic or pessimistic in the bilevel context. In this paper, we are considering the optimistic formulation, i.e. in case there are two or more set of optimal model parameters for a given hyperparameter that lead to a similar training loss, then the one that is better in terms of validation loss is preferred. Below we provide the optimistic formulation for the bilevel hyperparameter optimization problem.
Refer to Table~\ref{Table1:Notations} for the notations used throughout the paper.

\begin{align}\label{mod:originalBilevel}
\begin{split}
\minimize_{\lambda,\beta} \quad & F(\beta) \\
\st\quad  & \\
 & \hspace{-12mm} \beta \in \argmin_{\beta} \{f(\lambda,\beta)\}
 \end{split}
\end{align}

Here,  $F: \mathbb{R}^{m} \rightarrow \mathbb{R}$ and $f: \mathbb{R}^{n} \times \mathbb{R}^{m} \rightarrow \mathbb{R}$ are the upper and lower level objective functions, respectively, and $\lambda\in \mathbb{R}^n$ and $\beta\in \mathbb{R}^m$ are the corresponding variables. 

A Bilevel formulation in hyperparameter optimization minimizes the training set loss at the lower level and validation set loss at the upper level. 
The upper level validation loss will be minimized only for the set of model parameters that minimize the training set loss. The upper level objective is $F(\beta) = L(\beta;D_{val})$, where $D_{val}$ is the validation set and $L$ is the loss function. We use cross-entropy loss for our experiments, along with $L2$ regularization, also known as \textit{weight\ decay} in the context of neural networks. The lower level objective is as follows:
\begin{align}\label{Obj:Lower}
f(\lambda, \beta) = L(\beta; D_{tr}) + \Theta(\lambda, \beta)    
\end{align}
where loss, $L$, is defined on the training set, $D_{tr}$, and $\Theta(\lambda, \beta) = \lambda\|\beta\|^2$ is the $L2$ regularization term. The hyperparameter $\lambda\geq 0$ controls the penalty on the training loss function, which improves the learning of the model for unseen data by preventing overfitting on the training examples. 
Next, we discuss some common methods from the bilevel optimization literature that are employed to solve (\ref{mod:originalBilevel}) using single level reduction, and the drawbacks of these methods in the context of hyperparameter optimization in machine learning.

\subsection{Single Level Reduction using Reaction Set Mapping} 

The reaction set mapping in bilevel optimization provides the optimal $\beta$ with respect to the lower level problem for the corresponding $\lambda$. If there are multiple global optimal $\beta$s for a given $\lambda$, the mapping is expected to return all these solutions. The mapping is given by, \begin{align}\label{set_valued_map}
\Psi(\lambda) = \argmin_{\beta}\{f(\lambda,\beta)\}
\end{align} 
Single level reduction of the bilevel program depends on the approximation of (\ref{set_valued_map}). However, this mapping is not readily available, so we need to solve the lower-level problem for various values of the upper level variable vector, $\lambda$, to obtain (\ref{set_valued_map}). As there can be more than one optimal solutions at the lower level, the $\Psi(\lambda):\reals^m\tos\reals^n$ mapping is ideally a set-valued mapping. This is generally true in case of machine learning problems because of the non-convexity in the problem architecture. The single level reduction of the bilevel problem (\ref{mod:originalBilevel}) using (\ref{set_valued_map}) is given as follows:
\begin{align}\label{mod:reaction_set}
\minimize_{\lambda,\beta} \{F(\beta)\ :\ \beta \in \Psi(\lambda)\}
\end{align}
Studies by \cite{lorraine2018stochastic} and \cite{mackay2019self} utilize this reaction set mapping in the context of hyperparameter optimization. 
Both the studies assume (\ref{set_valued_map}) to be a single-valued mapping for their approximations, as it is mathematically and computationally hard to approximate the actual set-valued mapping $\Psi(\lambda)$.

\subsection{Single Level Reduction using First-order Conditions}
Another method to reduce a bilevel program into a single level problem is using the first-order conditions of the lower level problem. The reduced single level problem in such cases is the following constrained optimization problem:
\begin{align}\label{mod:first_order}
\minimize_{\lambda,\beta}\{F(\beta)\ :\ \nabla_{\beta} f(\lambda,\beta) = 0 \}
\end{align}
As machine learning problems are highly non-linear with multiple local optimums and saddle points, using a first-order reduction usually will not lead to the actual solution of (\ref{mod:originalBilevel}). Another problem with the above method is that solving the problem requires computation of the second-order derivatives. It is not feasible to compute second-order derivatives for large problems involving millions of model parameters that are fairly common in practice. The study by \cite{mehra2019penalty} relies on first-order conditions of the lower level for single level reduction of the original bilevel problem.

\section{Proposed Algorithm}\label{section:3}
We propose an alternative method in this section that uses the optimal value function mapping for single level reduction of the original bilevel problem. The optimal value function mapping $\phi(\lambda): \mathbb{R}^{n} \rightarrow \mathbb{R}$ for problem (\ref{mod:originalBilevel}) is given as:
\begin{align}\label{optimal_value_fun}
\phi(\lambda) = \minimize_{\beta}\{ f(\lambda,\beta) =  L(\beta; D_{tr}) + \Theta(\lambda, \beta) \}
\end{align}
The above function gives the optimal function value of the lower level problem corresponding to different values of $\lambda$. Using (\ref{optimal_value_fun}), the original problem (\ref{mod:originalBilevel}) can be expressed as follows:
\begin{align}\label{mod:bilevel_OVF}
\minimize_{\lambda,\beta}\{ F(\beta)\ :\ f(\lambda,\beta) \leq\phi(\lambda)\}
\end{align}
Note that unlike $\Psi(\lambda)$, $\phi(\lambda)$ is always a single-valued mapping, making it relatively simple to approximate. The proposed method iteratively approximates this optimal value function using Gaussian Process Regression that we discuss next.



\begin{figure*}[!t]
\centering
\subfloat[]{\includegraphics[width=3in]{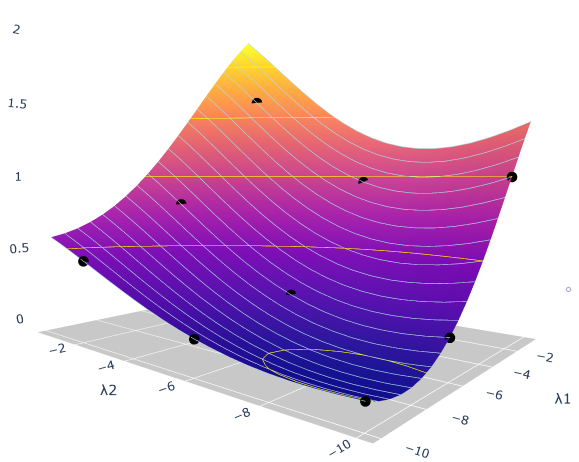}%
\label{fig:GPR0}}
\hfil
\subfloat[]{\includegraphics[width=3in]{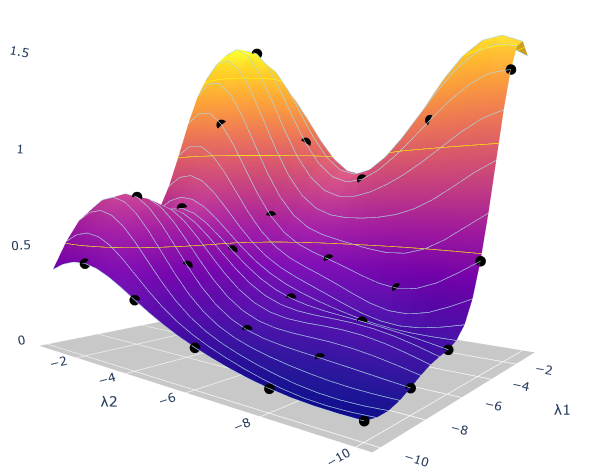}%
\label{fig_second_case}}
\caption{ $\phi$-mapping approximations using GPR for MNIST-10K in case of two hyperparameters. Here $(\lambda_1,\lambda_2)\rightarrow\phi(\lambda_1,\lambda_2)$ with $9$ and $25$ samples for (a) and (b), respectively.}
\label{fig:GPR1}
\end{figure*}

\subsection{Gaussian Process Regression}
Gaussian process regression (GPR) is used for interpolation of functions where the interpolation is governed by Gaussian processes (GP). We approximate the $\phi(\lambda)$ mapping using GPR with a sample $S = \{\lambda^{(i)}\}_{i=1}^t$ of hyperparameters in the feasible region. The lower level problem (\ref{Obj:Lower}) is solved for each $\lambda^{(i)}$ using stochastic gradient descent and the value of $\phi(\lambda^{(i)})$ is recorded. The resulting set $P = \{\lambda^{(i)}, \phi(\lambda^{(i)})\}_{i=1}^t$ is used for learning the model. 
The model structure that is used in a stochastic process is given as follows:
\begin{align}\label{kriging}
\phi(\lambda^{(i)}) = \mu + \epsilon(\lambda^{(i)}) \quad i = 1, \ldots, t,
\end{align}
where $\mu$ is the prior's expectation, and the random error $\epsilon(\lambda^{(i)})$ follows a standard normal distribution with variance $\sigma^2$. We consider the \textit{radial-basis function} (RBF) as the correlation function for the prior because of the kernel's stationarity. The RBF is given as follows:
\begin{align}\label{rbf_kernel}
\kappa(\lambda^{(i)},\lambda^{(j)}) = exp\bigg(-\frac{d(\lambda^{(i)},\lambda^{(j)})^2}{2} \bigg)
\end{align}
where $d(\lambda^{(i)},\lambda^{(j)})$ is as follows:
\begin{align}\label{ker_dist}
d(\lambda^{(i)},\lambda^{(j)}) = \sum_{k=1}^{n} \frac{\vert \lambda_{k}^{(i)}-\lambda_{k}^{(j)}\vert}{l_k}^2
\end{align}

Here, $l_k$ defines the length-scale of the respective upper level variable's dimension. The kernel defined in ($\ref{rbf_kernel}$) is a particular case of the class of correlation functions used in GP approximation. Note that the error terms have a negative correlation that decreases exponentially on increasing the distance between respective points. The optimal value of the parameters $\mu, \sigma, l_1,l_2...,l_n$ can be obtained by maximizing the log-likelihood of the sample. Figures~\ref{fig:GPR0} and~\ref{fig:GPR1} show how the GPR models change on adding more data points to initial sample $P$. For a detailed study on the stochastic process model, refer to \cite{jones1998efficient} and \cite{sacks1989design}. 

Let $\hat\phi(\lambda)$ and $\hat{s}(\lambda)$ be the approximate mapping and the standard error for $\phi(\lambda)$, respectively, obtained using GPR with the initial sample $P$, then a relaxation to the original problem (\ref{mod:bilevel_OVF}) can be written as follows:
\begin{align}\label{mod:final0}
\minimize_{\lambda,\beta}\{ F(\beta)\ :\ f(\lambda,\beta) \leq\hat{\phi}(\lambda) + z \hat{s}(\lambda)\}
\end{align}
Choosing a large enough $z$ value will ensure a lower bound to the problem ($\ref{mod:bilevel_OVF}$), which is discussed in the convergence results provided in Section~\ref{convergence_results}. When the GP model is accurate, $\hat{s}(\lambda)$ converges to zero and the optimal value function constraint in the above problem can be treated as an equality. We will reduce the problem (\ref{mod:final0}) into an unconstrained problem by using the augmented Lagrangian method, which is discussed in the next section.

\subsection{Augmented Lagrangian Formulation}
The Lagrangian method for solving optimization problems with equality constraints is widely used in the field of optimization. An extension of this method is known as the augmented Lagrangian method, which adds an extra quadratic penalty on the Lagrangian of the original problem. Without loss of generality, we will assume the constraint in problem (\ref{mod:final0}) holds with equality at the bilevel optimum. The unconstrained penalized formulation of (\ref{mod:final0}) using the augmented Lagrangian method is given as follows:
\begin{dmath}\label{eq:uncon}
\minimize_{\lambda,\beta} \mathcal{A}(\lambda,\beta) = F(\beta) + \frac{\rho}{2}\Big(\hat{\phi}(\lambda) + z\hat{s}(\lambda) - f(\lambda,\beta)\Big)^2  + \mu\Big(\hat{\phi}(\lambda)+z\hat{s}(\lambda)-f(\lambda,\beta)\Big)
\end{dmath}
The problem (\ref{eq:uncon}) is solved by optimizing the penalized loss function $\mathcal{A}$ iteratively, and updating the multipliers after each step. The update rule for $\rho$ and $\mu$ are discussed in the pseudocode of our approach in Algorithm $\ref{alg:pseudocode}$.
The unconstrained problem (\ref{eq:uncon}) can be solved using gradient descent-based methods. An optimal pair, $(\lambda^*,\beta^*)$, of upper and lower level variables is obtained after every iteration of the augmented Lagrangian method. Using this, we add a new point $(\lambda^*, \phi(\lambda^*))$ into the initial sample $P$ and re-approximate $\hat{\phi}(\lambda)$ and $\hat{s}(\lambda)$ using the updated sample. We refer to our approach as Gradient-based Bilevel Hyperparameter Optimization (GBHO) in rest of the paper.

\begin{algorithm}[H]
\caption{Pseudocode for GBHO algorithm. Note that we use $\rho^0=2$, $\mu^0=2$, $\eta=1.5$ and $z = 3$ in our experiments. GPR is an acronym for Gaussian process regression.}\label{alg:pseudocode}
\begin{algorithmic}[1]
\STATE \textbf{Require :} Training set $D_{tr}$, validation set $D_{val}$ and  initial sample of hyperparameters, $S = \{\lambda^{(i)}\}_{i=1}^{t}$.\\
\STATE \textbf{Ensure :} Optimal model parameters and hyperparameters, $i.e.\ $ $\lambda^{*}, \beta^{*}$.\\
\STATE {\textbf{for $i = 1 \cdots t$}}
        \STATE \hspace{0.5cm} $\phi(\lambda^{(i)}) \Leftarrow \minimize_{\beta}\{ L(\beta, D_{tr}) + \Theta(\beta,\lambda^{(i)})\}$
        \STATE \hspace{0.5cm} $F(\beta^{(i)}) \Leftarrow L(\beta^{(i)}, D_{val})$
\STATE \textbf{end for}
\STATE $\hat{\phi}(\lambda) \Leftarrow GPR\big( P^0 = \{\lambda^{(i)},\phi(\lambda^{(i)})\}_{i=1}^t \big)$\
\STATE Let $I^0 \Leftarrow (\lambda^{(0)},\beta^{(0)})$, $s.t.\ $ $F(\beta^{(0)}) = \minimize_i \{F(\beta^{(i)})\}_{i=1}^{t}$
\STATE {\textbf{for $j=1 \cdots N$}}
        \STATE \hspace{0.5cm} $(\lambda^{j}, \beta^{j}) \Leftarrow \argmin_{(\lambda,\beta)} \mathcal{A}(\lambda,\beta)$, with $\mu^{j-1}$ and $\rho^{j-1}$. Use $I^{j-1}$ as starting point.
        \STATE \hspace{0.5cm} $I^{j} \Leftarrow (\lambda^{j}, \beta^{j})$
        \STATE \hspace{0.5cm} $\mu^{j} \Leftarrow \mu^{j-1} + \rho^{j-1}\big(\hat{\phi}(\lambda^j)+3\hat{s}(\lambda)-f(\lambda^j,\beta^j)\big)$
        \STATE \hspace{0.5cm} $\rho^{j} \Leftarrow \eta \rho^{j-1}$
        \STATE \hspace{0.5cm} $P^j \Leftarrow   P^{j-1} \cup (\lambda^{(j)},\phi(\lambda^{(j)})$ 
        \STATE \hspace{0.5cm} $\hat{\phi}(\lambda) \Leftarrow GPR\big(P^j \big)$
\STATE \textbf{end for}
\STATE $(\lambda^{*}, \beta^{*}) \Leftarrow I^{N}$
\end{algorithmic}
\end{algorithm}

\subsection{Termination and Convergence Results}\label{convergence_results}
For the convergence proof we assume that we are working with a global optimizer that guarantees the global optimal solution whenever a single level unconstrained optimization problem is given. Under this assumption, our 
proposed method will lead to the optimal set of the regularization hyperparameters that we show using the following theorems.
\begin{theorem}\label{theorem_1}
For the hyperparameter optimization problem (\ref{mod:bilevel_OVF}), the following approximate optimization problem provides a lower bound $\minimize_{\lambda,\beta}\{ F(\beta)\ :\ f(\lambda,\beta) \leq\hat{\phi}(\lambda) + 3 \hat{s}(\lambda)\}$ with a probability $99.74\%$.
\end{theorem}
\begin{proof}
Note that the Gaussian process optimization provides an approximation to $\phi(\lambda)$ with an expected value of $\hat{\phi}(\lambda)$ and standard deviation of $\hat{s}(\lambda)$ for any given $\lambda$. Under the normality assumption, $\phi(\lambda) \le \hat{\phi}(\lambda) + z \hat{s}(\lambda)$ with a probability $p(z)$, where $p(z)$ represents the $p$-value corresponding to a single-tailed test for a given $z$. For $z=3$, the value of $p(z)$ is $99.74\%$. Therefore the problem $$\minimize_{\lambda,\beta}\{ F(\beta)\ :\ f(\lambda,\beta) \leq\hat{\phi}(\lambda) + 3 \hat{s}(\lambda)\},$$ represents a relaxation to problem (\ref{mod:bilevel_OVF}). A relaxation to a given optimization problem always leads to the lower bound of the problem.
\end{proof}

\begin{theorem}\label{theorem_2}
Solving the problem (\ref{mod:final0}) iteratively using Algorithm \ref{alg:pseudocode} leads to the optimal solution to problem (\ref{mod:bilevel_OVF}) at termination when $z$ is sufficiently large.
\end{theorem}
\begin{proof}
The algorithm terminates when the following two conditions are met: 
\begin{enumerate}
    \item [C1]: $\hat{s}(\lambda^\ast) \le \delta$, which implies that $\vert \phi(\lambda^\ast) - \hat{\phi}(\lambda^\ast)\vert \le z \delta$ for a sufficiently large value of $z$ (say $z \ge 3$). This ensures that $\phi(\lambda)$ is accurately approximated at $\lambda^\ast$.
    \item [C2]: $\vert \hat{\phi}(\lambda^\ast) - f(\lambda^{\ast},\beta^{\ast})\vert \le \epsilon$, which implies that $\vert f(\lambda^{\ast},\beta^{\ast}) - \phi(\lambda^{\ast})\vert \le z\delta+\epsilon$.
\end{enumerate}
From Theorem~\ref{theorem_1}, $(\lambda^\ast,\beta^\ast)$ represents the lower bound to problem~\ref{mod:bilevel_OVF} and since the constraint $f(\lambda^\ast,\beta^\ast) - \phi(\lambda^\ast) \le 0$ is satisfied with an error $z\delta+\epsilon$, $(\lambda^\ast,\beta^\ast)$ represents a near optimal solution. 
\end{proof}
The above theorem can be extended to also show that a small change in $\lambda$ over iterations of Algorithm 1 when C2 is satisfied also implies that the $\lambda$ in the current iteration is near optimal. We avoid this theorem for brevity. This ensures that the algorithm will never stagnate at a wrong solution. In case it does, the solution is guaranteed to be the optimal solution.

\section{Experiments}\label{section:4}

This section includes our experiments on MNIST \cite{lecun2010mnist} and CIFAR-10 \cite{Krizhevsky09learningmultiple} datasets. MNIST dataset consists of $28\times 28$ grayscale images of handwritten digits from $0$-$9$. For this multi-classification problem we solve multiple test instances using MLP architecture with a single hidden layer for one and two regularization hyperparameters. CIFAR-10 dataset consists of $32\times32$ colour images with 10 classes. For this classification problem, we employ LeNet-5 CNN architecture \cite{lenet} which consists of 3 convolutional layers, 2 subsampling layers and 2 fully connected layers. Here also, we optimize multiple regularization hyperparameters corresponding to different layers. For all the experiments, the discrete hyperparameters, such as learning rate and momentum, are tuned over similar grid values for fair comparison. We ran our entire experiments on \textit{Google Colaboratory Pro Plus} cloud platform.

\subsection{MNIST Dataset}



\begin{table*}
  \caption{Performance comparison of GBHO with Random search, Grid search, Bayesian Optimization and HyperBand in case of MNIST dataset. 1HP denotes single hyperparameter test instances and 2HP denotes two hyperparameter test instances. TRL, VAL and TEL denote the training loss, validation loss and testing loss, respectively. LLO denotes the number of lower level optimization problems solved by the corresponding method. In case of GBHO, (+5) indicates the number of unconstrained optimization calls needed to solve the problem (\ref{eq:uncon}).}
  \label{sample-table}
  \begin{tabularx}{\textwidth}{X c | X X X | X X X }  
    \toprule
     &  & \multicolumn{3}{c}{\textbf{MNIST-1HP}} &  \multicolumn{3}{c}{\textbf{MNIST-2HP}}  \\
    \cmidrule(r){3-8}
    \textbf{\textbf{Methods}} &   & $\mathbf{1000}$ & $\mathbf{5000}$ & $\mathbf{10000}$ & $\mathbf{1000}$ & $\mathbf{5000}$ & $\mathbf{10000}$ \\
    \midrule
     \multirow{5}{4em}{\textbf{Grid Search}} & $TRL$ & 0.0216 & 0.0321 & 0.0302 & 0.0332 & 0.0685 & 0.0300\\ 
  & $VAL$ & 0.4395 & 0.3045 & 0.1524 & 0.4327 & 0.2757 & 0.1462\\
  & $TEL$ & 0.4125 & 0.2714 & 0.1942 & 0.4193 & 0.2556 & 0.1909\\
  & $\lambda$ & -6.06 & -6.36 & -6.77 & -0.54, -4.78 & -0.68, -1.03 & -0.68, -6.89\\
  & $LLO$ & 100 & 100 & 100 & 900 & 900 & 900\\
    \midrule
    \multirow{5}{4em}{\textbf{Random Search}} & $TRL$ &  0.0258 & 0.0534 & 0.0268 & 0.0204 & 0.0437 & 0.0328\\
  & $VAL$ & 0.4413 & 0.3145 & 0.1532 & 0.4403 & 0.2687 & 0.1438 \\
  & $TEL$ & 0.4177 & 0.2740 & 0.1963 & 0.4227 & 0.2419 & 0.1830\\
  & $\lambda$ & -5.85 & -5.88 & -6.90 & 0.0, -3.44 & -0.30, -2.11 & -0.23, -5.08\\
  & $LLO$ & 100 & 100 & 100 & 900 & 900 & 900\\
  \midrule
  \multirow{5}{4em}{\textbf{Bayesian Search}} & $TRL$ &  0.0379 & 0.0243 & 0.0059 & 0.0420 & 0.0218 & 0.0088\\ 
  & $VAL$ & 0.4396 & 0.2680 & 0.1318 & 0.4288 & 0.2558 & 0.1331\\
  & $TEL$ & 0.4278 & 0.2434 & 0.1824 & 0.4057 & 0.2360 & 0.1788\\
  & $\lambda$ &  -10.00 & -5.88 & -7.92 & -9.24,-3.55 & -6.57, -6.36 & -6.92, -10.00\\
  & $LLO$ & 60 & 60 & 60 & 100 & 100 & 100\\
  \midrule
  \multirow{5}{4em}{\textbf{HyperBand}} & TRL & 0.0277 & 0.0164 & 0.0033 & 0.0356 & 0.0100 & 0.0255 \\ 
  & $VAL$ &  0.4335 & 0.2637 & 0.1391 & 0.4236 & 0.2640 & 0.1336\\
  & $TEL$ & 0.4224 & 0.2436 & 0.1916 & 0.4063 & 0.2444 & 0.1717\\
  & $\lambda$ &  -5.43 & -6.48 & -8.61 & -5.3, -6.05 & -7.75, -7.14 & -6.68, -6.89\\
  & $LLO$ & 254 & 254 & 254 & 254 & 254 & 254 \\
  \midrule
  \multirow{5}{4em}{\textbf{GBHO}} & TRL & 0.0210 & 0.0483 & 0.0467 & 0.0298 & 0.0332 & 0.0114 \\ 
  & $VAL$ &  0.0279 & 0.0133 & 0.0281 & 0.0457 & 0.0745 & 0.0980\\
  & $TEL$ & 0.3734 & 0.2383 & 0.1765 & 0.3652 & 0.2035 & 0.1711 \\
  & $\lambda$ &  -6.61 & -2.75 & -3.95 & -0.11, -0.52 & -2.59, -1.35 & -10.26, 0.316\\
  & $LLO$ & 10 (+5) & 10 (+5) & 10 (+5) & 50 (+5) & 50 (+5) & 50 (+5) \\
  \bottomrule
  \end{tabularx}
  \label{table:MNIST}
\end{table*}

For MNIST dataset, we use the MLP architecture with one hidden layer consisting of 100 nodes. We perform our study on three instances consisting of 1000, 5000 and 10000 examples which we sample randomly from the original MNIST dataset. We split the instances into 60\% training and 40\% validation examples for the experiments. The test set used for the comparison consists of 10000 data points sampled separately for each of the cases. We solve two hyperparameter optimization problems for each set. For the first experiment, we regularize weights of both the hidden and output layer using a single hyperparameter. For the second experiment, we use two separate regularization hyperparameters to regularize each of the layers separately. The results of the experiments are reported in Table $ \ref{table:MNIST}$. For the single hyperparameter case, we sample 10 different $\lambda$ points from the interval $[-10,0]$ in order to approximate the $\phi$-mapping initially. We take $\exp(\lambda)$ as the regularization hyperparameter such that $\exp(\lambda)\geq 0$. Similarly, for the two hyperparameters we sample $25$ different $\lambda$ points to approximate the $\phi$-mapping initially.

\begin{figure*}[!t]
\centering
\subfloat[]{\includegraphics[width=3in]{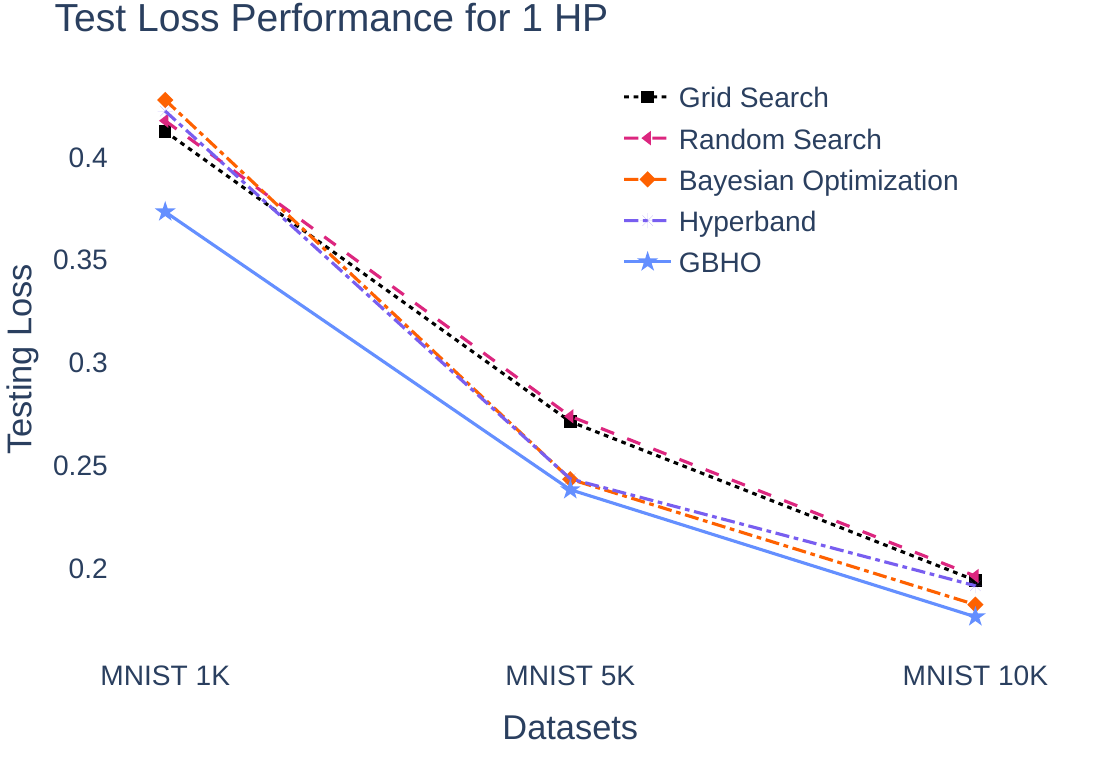}%
\label{fig:comparison_1hp}}
\hfil
\subfloat[]{\includegraphics[width=3in]{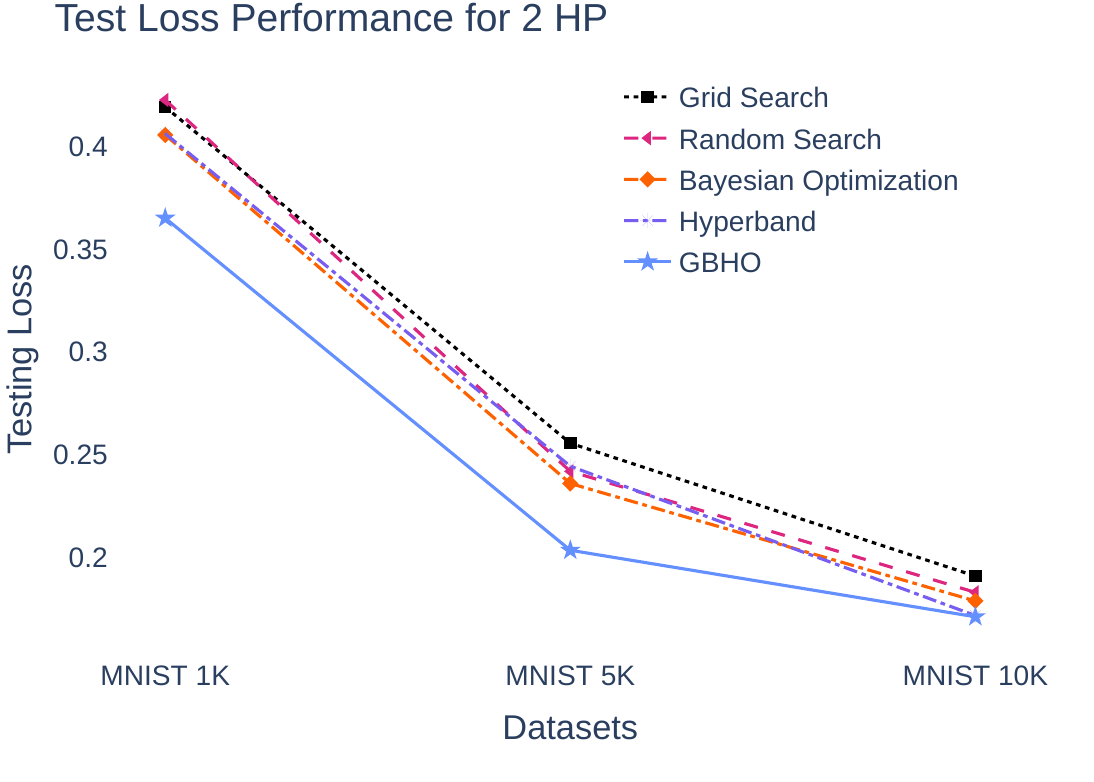}%
\label{fig:comparison_2hp}}
\caption{ Comparison of test set loss for random search, grid search, Bayesian optimization, Hyperband algorithm and GBHO for MNIST-1HP, (a), and MNIST-2HP, (b), cases.}
\label{fig:comparison}
\end{figure*}

Figures \ref{fig:comparison_1hp} and \ref{fig:comparison_2hp} compare the test set loss for the single and the double hyperparameter cases. We provide comparison of our method with random search, grid search, Bayesian optimization \cite{omalley2019kerastuner} and Hyperband algorithm \cite{hyperband}. We can clearly see from the figures that our method outperforms all the other techniques in terms of the test loss for all the cases. Table \ref{table:MNIST} contains a detailed comparison of our method with other techniques where we provide the training, validation and testing losses for the best model obtained from each method. The computational expense of each method is measured by counting the number of lower level optimization tasks executed by each method. Our method is computationally much faster as compared to other methods as it requires significantly lower number of lower level optimization calls.

\subsection{CIFAR-10 Dataset}
For CIFAR-10 dataset, we use LeNet-5 CNN architecture for our experiments, and once again compare with random search, grid search, Bayesian optimization and HyperBand method. We perform our study on a small batch of 1000 data points with $60\%$ training set and $40\%$ validation test to allow for overfitting. The results are compared on a testing dataset of $10000$ points which are sampled separately. We use two and four regularization hyperparameters for our test instances. In the case of two hyperparameters, we regularize the two convolutional layers and the three dense layers with each hyperparameter, while in the case of four hyperparameters, we regularize the two convolutional layers with one hyperparameter and each of the three dense layers with three separate hyperparameters. The results obtained are similar as in the context of MNIST dataset, where we observe that the proposed method is able to find the best testing loss with the least computational requirements. The test loss results have been reported through Figure~\ref{fig:CIFAR}. The lower level optimization calls for grid search, random search, Bayesian search, HyperBand method and GBHO are 100, 100, 100, 254, and 25 (+5), respectively, for the 2HP case. For the 4HP case, the lower level optimization calls for grid search, random search, Bayesian search, HyperBand method and GBHO are 625, 625, 100, 254, and 81 (+5), respectively. Interestingly, the results on test data deteriorate for the 4HP case as compared to the 2HP case for Bayesian optimization, HyperBand as well as GBHO, the reason for which could be overfitting on the validation set with increasing hyperparameters. In the case of MNIST dataset while moving from 1HP to 2HP we observed some benefits.

\begin{figure}[!t]
\centering
\includegraphics[width=3.4in]{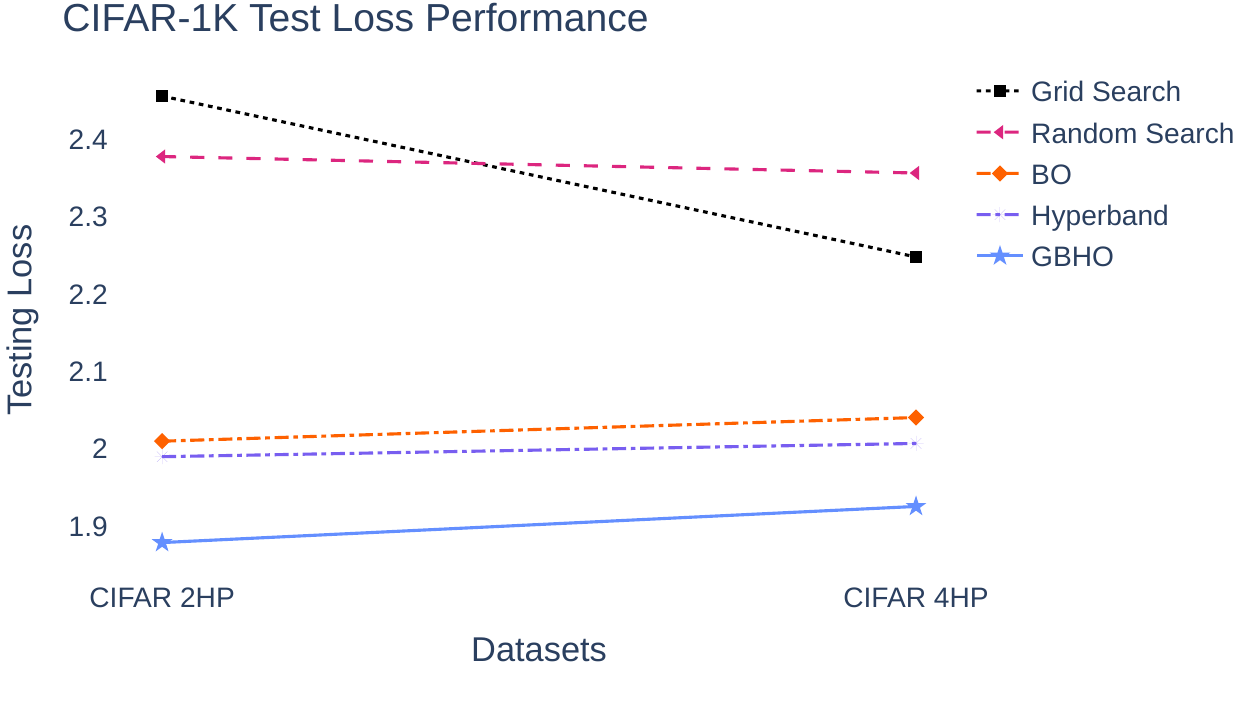}
\caption{Test set loss for CIFAR-10. BO present in the figure legends denotes Bayesian Optimization.}
\label{fig:CIFAR}
\end{figure}

\section{Conclusions}\label{section:5}

In this paper, we have proposed a Gradient-based Bilevel Hyperparameter Optimization (GBHO) method for optimizing continuous hyperparameters in machine learning models. We have performed our study in the context of regularization hyperparameters and have shown that the proposed method is able to optimize multiple hyperparameters with low computational requirements. The method formulates the hyperparameter optimization problem as a bilevel optimization task and then solves it by using the augmented Lagrangian method after reducing it to single level using the optimal value function approach. An extensive computational study with MNIST dataset on MLP architecture and CIFAR-10 dataset on CNN architecture shows that the method is able to optimize multiple hyperparameters with low computational requirements as compared to other techniques, such as, grid search, random search, Bayesian optimization and HyperBand method. The test loss of the models generated using the proposed algorithm also turns out to be better than the competing approaches.

\section{REFERENCES}

\bibliographystyle{IEEEtran}
\bibliography{bibliography}


 





\end{document}